\documentclass[twoside]{article}

\usepackage{template/style}

\usepackage{template/macro}

\usepackage[T1]{fontenc}
\usepackage[utf8]{inputenc}
\usepackage[english]{babel}

\usepackage{times}

\PassOptionsToPackage{hyphens}{url}
\usepackage{url}
\usepackage{hyperref}

\usepackage{graphicx}
\usepackage{caption}
\usepackage{tikz}
\usepackage{pgfplots}
\pgfplotsset{compat=1.17}

\usepackage{comment}

\usepackage{array}
\usepackage{booktabs}
\usepackage{microtype}
\usepackage{xfrac}
\usepackage{xcolor}
\usepackage{enumitem}
\setitemize{itemsep=0pt,parsep=0pt,topsep=0pt}
\setenumerate{itemsep=0pt,parsep=0pt,topsep=0pt}

\usepackage{natbib}

\usepackage[normalem]{ulem}
\usepackage{hhline}

\begin{document}
\twocolumn[
\mytitle{Learning with Hidden Factorial Structure}
\myauthor{
Charles Arnal$^*$ \And Clement Berenfeld$^*$ \And Simon Rosenberg$^*$ \And  Vivien  Cabannes$^*$}
\myaddress{Meta AI, FAIR \And Postdam University \And C3AI \And Meta AI, FAIR} ]

\def\thefootnote{*}\footnotetext{The authors contributed equally to this work.}\def\thefootnote{\arabic{footnote}}

\begin{abstract}

Statistical learning in high-dimensional spaces is challenging  without a strong underlying data structure. 
Recent advances with foundational models suggest that text and image data contain such hidden structures, which help mitigate the curse of dimensionality. 
Inspired by results from nonparametric statistics, we hypothesize that this phenomenon can be partially explained in terms of decomposition of complex tasks into simpler subtasks.
In this paper, we present a controlled experimental framework to test whether neural networks can indeed exploit such ``hidden factorial structures.''
We find that they do leverage these latent patterns to learn discrete distributions more efficiently.
We also study the interplay between our structural assumptions and the models' capacity for generalization.

\end{abstract}

\section{Introduction}

\paragraph{Context and motivations}
The ability of artificial intelligence systems to solve highly complex tasks by extracting information from a vast amount of data remains poorly understood.
On the one hand, the \emph{curse of dimensionality} states that learning becomes virtually impossible as the data dimension grows large. On the other hand, it has been now long observed that these systems can develop a fine understanding of even high-dimensional data. This is for instance the case for Large Language Models (LLMs), which process lists of tokens whose different combinations can easily exceed $10^{80}$ in number \citep{brown2020language}.
It is well-known that the curse of dimensionality can be overcome when the data exhibits a lower-dimensional underlying structure, and if the learning procedure can adapt to this intrinsic dimension.
Many structural assumptions have been proposed by statisticians, such as the existence of a low-dimensional manifold on which the data lies \citep{fefferman2013testing}, smoothness of the functional to learn \citep{caponnetto2007optimal} or sparsity of the later \citep{hristache2001structure}.
A recent line of research has argued that the factorization of a complex task into simpler sub-tasks could better explain the ability to learn such a complex task efficiently \citep{parascandolo2018learning,arora2023theory,ahuja2024provable,Cagnetta2024,CompositionalRisk}.
We approach this work from this perspective.

\paragraph{Contributions}
We adopt a discrete data framework, which is a reasonable choice given the two following considerations: first, the frequently discrete nature of data (in particular text data); and second, the operational basis of LLMs, which rely on tokenized inputs and are of particular interest in contemporary research, and one of the main motivations behind this work.
Within this framework, we propose a factorization-based model where both the input and output spaces are decomposed into a products of small unknown factors, and where the tasks are structured along this factorization. This data structure is motivated by examples detailed further below. Our goal is to {\em test whether neural networks can leverage hidden factorial structure to learn more efficiently}, both in terms of computational and statistical complexity. 
%
Due to the lack of theoretical tools at hand, we focus on an empirical approach through a controlled exploratory study that connects the data structure and the performance of Multilayer Perceptrons (MLPs). We choose to study MLPs as they are the essential building blocks of many modern machine learning techniques, and in particular of LLMs, where they represent the majority of the model's parameters. 
By extensively studying how the performances of our model react to changes in train data size, number of parameters, compute resources, and complexity of the hidden data structure, we hope to shed new light on some of the core principles behind learning in the LLM era.
To summarize our contributions:
\begin{enumerate}[leftmargin=1em]
    \item We propose a new type of structural assumptions on discrete data distributions; 
    \item We provide theoretical insights into the impact of these assumptions on the difficulty of learning the distributions; 
    \item We provide experimental proofs that Neural Networks can leverage our hypotheses to learn more efficiently. 
\end{enumerate}

\paragraph{Related works}
Our assumptions are inspired both by classical structural approaches from nonparametric statistics \citep{hristache2001structure,fefferman2013testing} and by recent works in machine learning suggesting the importance of compositionality and sub-task decomposition to understand and improve learning \citep[e.g.,][]{dziri2023faithfatelimitstransformers,alabdulkareem2024securellmusingcompositionalitybuild,valvoda2023benchmarkingcompositionalityformallanguages,wang2024grokkedtransformersimplicitreasoners,wies2022sub,zhang2024memory,guo2024causal, CompositionalRisk}, as well as hidden sparsity \citep[e.g.,][]{Barak2022HiddenPI,marion2024attentionlayersprovablysolve}. 

The adaptivity of neural networks to sparsity assumptions is an active area of research, with existing results holding only under restrictive conditions \citep[simple optimizers, vanilla architectures, or asymptotic regimes, see e.g.,][]{arnaboldi2024repetita, lee2024neural}. The main theoretical results in this area are linked to either multi-index models \citep[e.g.,][]{bietti2023learning, damian2024computational} or the sparse parity problem  \cite{daniely2020learning}. These frameworks are quite different from the settings typically discussed by practitioners training large LLMs, who often describe LLMs as processing discrete information --- which motivates our discrete setting.

\section{Setting}\label{sec:setting}

In this paper, we make Neural Networks (NNs) learn discrete conditional distributions $p(y|x)$ and we analyze both theoretically and experimentally how specific structural assumptions about these distributions can make the task easier.

More precisely, we consider a set of inputs $\cX$ of cardinality $N \in \N$ and a set of outputs $\cY$ of cardinality $M\in \N$.
We assume that the input/output pairs are generated from a joint distribution $(X,Y) \sim p$ on  $\cX\times\cY$, and we task a NN with learning the conditional distribution $p(y|x)$ for each input $x\in\cX$.
The quality of a learned estimator $\widehat p$ is measured through the Kullback-Leibler divergence, defined as
\begin{equation}\label{eq:loss}
    {\cal L}(\widehat p) = \E_{(X,Y)\sim p}\left[-\log \frac{\widehat p(Y | X)}{p(Y|X)}\right].
\end{equation}
From an optimization point of view, this loss is equivalent to the cross-entropy loss used by practitioners.
Indeed, the loss ${\cal L}$ is the excess risk of the cross-entropy loss, i.e. the cross-entropy loss minus its minimizer.

In order to learn from discrete data, we embed our problem in a continuous space $\R^d$ with two (learnable) embeddings:
\begin{equation}
 \label{eq:io}
    e: \cX \to \R^d, \qquad u:\cY \to \R^d.
\end{equation}
In addition to the embeddings, the NN learns a transformation $F:\R^d \to \R^d$ of the embedding space.
The final estimator is parameterized through the softmax function
\begin{equation}\label{eq:functionalform}
   \widehat p(y|x) = \frac{\exp(u_y^\top F(e_x))}{\sum_{y'\in\cY}\exp(u_{y'}^\top F(e_x))}, 
\end{equation}
where we use the abbreviations $e_x := e(x)$ and $u_y := u(y)$.
This setting, and in particular the fact that $e_x$ and $u_y$ belong to the same embedding space, was designed to closely match current state-of-the-art architectures, where residual connections are commonly used \citep{he2015deep}. 

Among all possible conditional density functions $p(y|x)$, we are particularly interested in those that satisfy certain structural conditions, which we define and motivate thereafter. Our goal is to establish whether NNs can efficiently leverage those assumptions to better learn and represent $p(y|x)$ in two slightly different settings: one in which the embeddings $e_x$ are learned, and one in which fixed, factorization-compatible embeddings are used.

\subsection{The Factorization Hypothesis}\label{subsec:factorization}

Our factorization hypothesis assumes that both the input space $\cX$ and the output space $\cY$ can be decomposed into $k\in\N$ and $\ell\in\N$ {\em unknown factors} respectively.
In other terms,
\[
    \cX \cong \prod\nolimits_{i \in [k]} \cX_i \quad \text{and} \quad \cY \cong \prod\nolimits_{j \in [\ell]} \cY_j,
\]
where each $\cX_i$ and $\cY_i$ is a discrete space containing $p_i = \card{\cX_i}$, respectively $q_j = \card{\cY_j}$, elements.
By ``unknown'', we mean that the mappings $\cX \cong \prod \cX_i $ and $\cY \cong \prod \cY_j$ are not given to the agent (the NN) trying to learn the conditional distribution $p(y|x)$.
We further assume that for all index $j \in [\ell]$, there exists a subset $I_j \subseteq [k]$ such that the conditional probability distribution $p(y|x)$ factors into
\begin{equation} \label{eq:prod} \tag{FH}
p(y|x) =  \prod\nolimits_{j \in [\ell]} p(y_j | \pa_j),
\end{equation}
where
\[
    \pa_j = (x_i)_{i\in I_j} \in \prod\nolimits_{i\in I_j} \cX_i.
\]
Here, $x_i\in \cY_j$ denotes the $i$-th coordinate of $x$ in its factor decomposition (respectively $y_j\in \cY_j$ is the $j$-th coordinate of $y$), hence $\pa_j$ is the set of coordinates that influence the factor $y_j$.
In other words, the Factorization Hypothesis~\eqref{eq:prod} states that the coordinates $y_j$ of $y$ in the decomposition are independent given the input $x$, and that each $y_j$ only depends on the subset $\pa_j$ of the input coordinates.
This model allows for structures where only a (potentially small) subset of hidden characteristics of the input influences each feature of the output.
As a consequence, learning the task $p(y|x)$ can be done by learning the subtasks $p(y_j | \pa_j)$.

The notation $\pa_j$ stands for ``parents'' and is borrowed from the field of graphical probabilistic model \citep{jordan2004}, as our problem can naturally be represented by a bipartite directed acyclic graph, where a first set of $k$ vertices represents the factors of $x$ and a second set of $l$ vertices represents those of $y$ and an edge between $x_i$ and $y_j$ indicates a non-trivial causal relationship between the two random variables, see Figure~\ref{fig:dag}.

\begin{figure}[h!] 
\centering
\includegraphics[scale = .8]{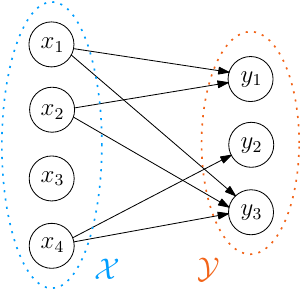}
\caption{A graphical representation of a model for $k = 4$ and $\ell = 3$. In this example, $\pa_1 = (x_1,x_2)$, $\pa_2 = (x_4)$, $\pa_3 = (x_1,x_2,x_4)$, and the third hidden variable $x_3$ has no impact on the output $y$.}
\label{fig:dag}
\end{figure}

\begin{remark}[Compression and representation]~
    Each layer of an MLP tasked with approximating a map typically performs a mix of representation, i.e. applies some information-preserving transformation to its input to make it more suited to later computations, and compression, i.e. processes the information in a many-to-one fashion.
In our setting, factorizing $\cX$ corresponds to  representation and mapping (probabilistically) the factors of $\cX$ to those of $\cY$ to compression.
\end{remark}

The following examples provide motivations for the factorization hypothesis.

\begin{example}[Text data]
Text data implicitly factors into characteristics of small cardinality, such as e.g. language, grammatical function, length, sentiment, complexity, etc. With most tasks, each of the characteristics of the expected output only depends on a subset of those of the input: in a next-word prediction setting, for example, the language of the next word would be the same as that of the input, its grammatical function would depend on the grammatical functions and language of the previous words, etc.
\end{example}

\begin{example}[Recommender systems] 
Likewise, consider a recommender system tasked with matching users with targets (e.g. movies on a streaming platform, products on a marketplace, etc.). One expects that the relevant data for both users and targets factor in such a way that only a small number of (potentially unknown) characteristics of a given user (e.g. nationality, language, socioeconomic status, etc.) impacts each factor of the recommended target (e.g. language, movie genre, etc.).
In this context, our setting is to be compared with low-rank matrix completion approaches \citep{candes2010matrix} where users and target are implicitly assumed to be embedded in a low-dimensional space (each dimension corresponding to a factor) and that preferences are measured through inner products (i.e. each factor of the target is only affected by one of the user's factors).
\end{example}

\begin{remark}[Link to mixture estimation]~
The law of the variable $y$ as defined above can be written as 
    $
    p(y) = \sum
    \alpha_\bz \prod
    p(y_j | \pa_j = \bz_j)
    $
    where $\alpha_\bz = \bP(\pa_1 = \bz_1,\dots,\pa_\ell = \bz_\ell)$. Thus learning $p(y|x)$ involves in particular learning the components of a mixture of products of distributions over a discrete space \citep{feldman2008learning}, with the added difficulty that one does not know a priori the product decomposition on $\cY$.
\end{remark}

\subsection{Input Embeddings}
\label{subsec:embeddings}

This subsection discusses two possible embeddings of our discrete problem in a continuous space, which are then used as inputs to the transform  $F$: a learned encoding, and a fixed factorization-compatible embedding. Both are defined further below.

\paragraph{Learned embedding}
In this classic setting, used e.g. in most LLMs, \citep[see ][]{brown2020language,touvron2023llama}, each discrete element $x\in \cX$ is mapped to a learned vector $e_x \in \R^d$.
The embeddings are typically initialized as random Gaussian vectors, and learned during training jointly with the embedding transform $F$ and the output embedding $u$.
Though very natural, this setting lacks a crucial property: as no structure is enforced on the embeddings, a trained NN has no hope of correctly generalizing to an input $x$ that was not part of its training set, and whose embedding would simply be equal to its random initialization.
This stands in contrast to the factorization-compatible embeddings described below.

\paragraph{Factorization-compatible embedding}
In this setting, we assume that an oracle or a previously trained NN provides us with an embedding $\tilde e_x$ that is adapted to the (unknown) factorization of $\cX$ in the following sense:
\begin{equation}
    \label{eq:fce}
    \tilde e_x = [E_1 \hdots E_k] 
    \begin{bmatrix}
    \ind{x_1} \\
    \vdots \\
    \ind{x_k}
    \end{bmatrix} = \sum_{i\in[k]} e^i(x_i),
\end{equation}
where $E_i \in \R^{d\times p_i}$ is some fixed matrix,  $\ind{x_i} \in \R^{p_i}$ is the one-hot encoding of $x_i \in \X_i $ and  $e^i(x_i) = E_i \ind{x_i}  \in \R^d$.
In other words, each discrete element $x_i \in \cX_i$ is mapped to some fixed, arbitrary embedding $e^i(x_i)$. 
The total embedding $\tilde e_x$ of $x$ is then the sum of the embeddings of its factors.
See Figure~\ref{fig:emb} for an illustration.

\begin{figure}[t]
    \centering
    \includegraphics[width=0.6\linewidth, page = 2]{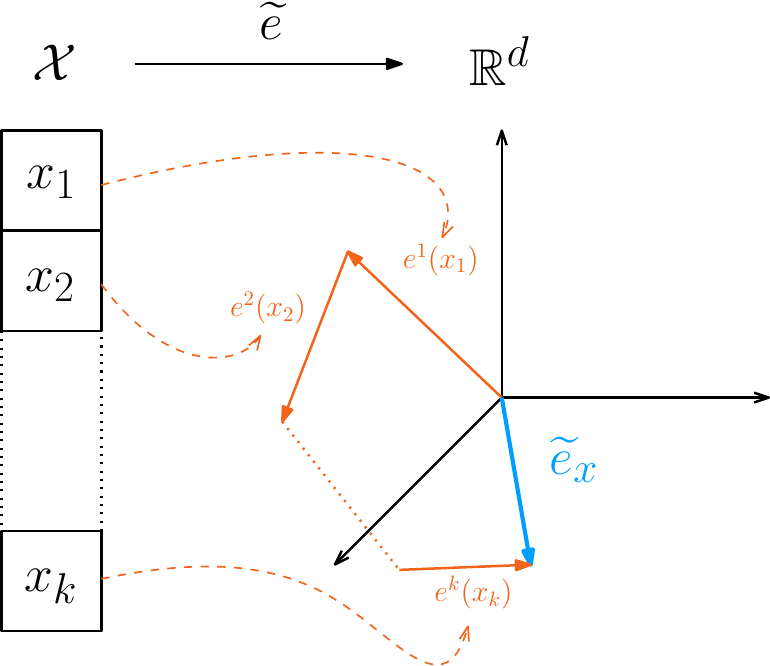}
    \caption{A diagram of a factorization-compatible embedding $x \mapsto \tilde{e}_x$.}
    \label{fig:emb}
\end{figure}

Note that when $d\geq \sum_{i\in [k]} p_i$, any matrix $[E_1 \hdots E_k]$ with coefficients sampled from continuous distributions will almost surely be of rank $\sum_{i\in [k]} p_i$.
As a consequence, the embedding $\tilde e$ will be linearly invertible, in the sense that learning a simple linear transformation would enable the NN to recover the factors $(x_1,\ldots, x_k)$ from its input $e_x$.
This transformation could be learned from $\sum p_i$ examples, which can be much smaller than the cardinality of the full set $\cX$ of $\prod p_i$ elements (as small as $\log N$).
Note that, thanks to its nonlinearities, a NN may retrieve these factors even when $d < \sum p_i$.

In contrast to learned embeddings, factorization-compatible embeddings allow for generalization across inputs.
Indeed, consider the simple case where $\cX \cong \cX_1 \times \cX_2$, $\cY \cong \cY_1 \times \cY_2$ and $p(y|x) = p(y_1|x_1)p(y_2|x_2)$.
Assume that the NN has learned to invert the embedding $\tilde e$, and can recover the factors $(x_1,x_2)$ of $x$  given an input $\tilde e_x$. 
Then the conditional distribution $p(y|x)$ for some yet unseen input $x = (x_1,x_2)$ can be estimated from past observations as long as points of the shape $(x_1,x_2')$ and $(x_1',x_2)$, for some $x_1' \in \cX_1, x_2'\in \cX_2$, were seen before --since $p(y_1|x_1)$ and $p(y_2|x_2)$ can be computed from $(x_1,x_2')$ and $(x_1',x_2)$ respectively.
Likewise, given a more complicated factorization, the conditional distribution $p(y|x)$ can be estimated for a yet unseen $x$ as long as each of the parent coordinates $\pa_j = (x_i)_{i\in I_j}$ have already appeared in previously observed data points.

The examples below provide motivations for this setting.

\begin{remark}[Link to transformers]
The functional form of the factorization-compatible embeddings closely matches that of the outputs of the attention layers in transformers, which are produced as a weighted sum of value vectors. Those outputs are then used as inputs to MLPs.
\end{remark}

\begin{remark}[Link to other statistical problems]~
When the images of the maps $(e^1,\dots,e^k)$ are linearly free, there exists some (unknown) linear mapping $(T_1,\dots,T_k)$ such that $T_i(e_x) = e^i(x_i)$ for all $i \in [k]$. When furthermore $e^i$ is injective over $\cX_i$, the target distribution can be expressed as
\[
    p(y|x) = \prod\nolimits_{j \in [\ell]} p(y_j|\pa_j)
= \prod\nolimits_{j \in [\ell]} p(y_j| \mathrm{Pa}_j(e_x)),
\]
where
\[
  \mathrm{Pa}_j: z \in \R^d \mapsto (T_i(z))_{i \in I_j} 
\]
is a linear map (the Parent map). 
If we now assume that $\ell = 1$ and that we are given an embedding $u : \cY \to \bR^{d}$, the model takes the simpler form
$
p(u|e) = p(u|\Pa(e)).
$
When the distribution $p(y|x)$ is close to a Dirac distribution, i.e. when the assignment $x \mapsto y$ can be seen as a deterministic function with some added noise, the situation can be reframed as a regression problem
$u = f(\Pa(e)) + \mathrm{noise}$,
which coincides with the well-studied multi-index problem \citep{hristache2001structure}. In particular, it has been shown that even with $\Pa$ unknown, the rate of estimation of the regression function depends on the rank of $\Pa$ and not on the embedding dimension of $e$. 
Closely related settings involving hidden sparsity include \citep{marion2024attentionlayersprovablysolve} and \citep{Barak2022HiddenPI}.
\end{remark}

\section{Theoretical Analysis}\label{sec:theory}

We consider both the computational and statistical difficulty of learning a factorizable data distribution  --that is to say both the minimal theoretical number of parameters needed in an MLP and the minimal theoretical number of training points needed for this task. For a given probability distribution $p$ on $\cX \times \cY$, and a given factorization $\cF$ of $\cX$ and $\cY$, we let 
$$
p_\cF(x,y) := p(x) \prod_{i=1}^\ell p(y_j|\pa_j).
$$
With this notation, the Factorization Hypothesis \eqref{eq:prod} is equivalent to stating that $p = p_\cF$ for some factorization $\cF$. We let $\cG(p)$ be the set of all such factorizations (which can be limited to the trivial factorization in the worst case).

\subsection{Approximation Complexity} \label{sec:approx}

We aim to assess the impact of our factorization hypothesis on the computational complexity involved in approximating the conditional distributions $p(y|x)$ for our models. We let $\cF$ be a factorization compatible with $p$.

Consider first  the learned embeddings defined in the previous subsection.
Let $E \in \R^{d\times N}$ and $U \in \R^{d\times M}$ be the learned embedding matrices that map the one-hot encoding of $x$ to $e_x$, respectively the one-hot encoding of $y$ to $u_y$, and define the matrix $G:= F(E) \in \R^{d\times N}$, where the transform $F$ is applied column-wise (i.e. to each $e_x$). 
Note that as each token $x$ can be mapped to any vector $e_x \in \R^d$ by the embedding process, the transform $F$ adds nothing to the expressivity of the overall model (though it might impact the training dynamics).
Constrained by the functional form of our estimator stated in~\eqref{eq:functionalform}, we want to represent $p(y|x)$ as proportional to $\exp(u_y^\top F(e_x))$.
This is equivalent to representing the matrix $\Lambda := (\log p(y|x))_{y,x} \in \R^{M\times N}$ as a matrix product 
\begin{equation}\label{eq:matrixEq}
   \Lambda = U^\top G + \text{column-wise constant},
\end{equation}
where the column-wise constant accounts for the renormalizing factor from~\eqref{eq:functionalform}.
\begin{theorem} \label{thm:ac}
    The log-likelihood matrix $\Lambda$ of the conditional probability distribution $p(y|x)$ can be factorized as $\Lambda = U^\top G$, where $U \in \bR^{\bar\chi(p) \times M}$, $G \in \bR^{\bar\chi(p) \times N}$ with
\begin{equation}
    \label{eq:AC}
    \tag{AC}
    \bar\chi(p) := \min_{\cF \in \cG(p)} \sum_{j\in [\ell]}\min \{ |\pa_j|,  q_j\},
\end{equation}
and
\[
    |\pa_j| := \prod_{i\in I_j} |\cX_i|.
\]
\end{theorem}
See the Supplementary Materials for the proof of this result.
It shows that in the learned embedding setting, the distribution can be perfectly approximated using an embedding dimension $d = \bar\chi(p)$ and $(M + N)\cdot d $ parameters, and that the factorization hypothesis may enable an {\em exponential gain} regarding the computational difficulty of the task at hand, going from $d = \min(M, N)$ when no structure is assumed on the data, to $d$ potentially as small as $\min\{\log_2M, \log_2N\}$.

Note that while the quantity $\bar \chi(p)$ bounds the rank of $\Lambda$, it is easy to show that $\Lambda$ having a small rank does not imply the existence of a factorization.  Simple counter-examples are provided in the Supplementary Materials.

Though the situation is slightly more complex than in the case of learned embeddings, similar results hold in the case of factorization-compatible embeddings: the number of parameters and the embedding dimension needed to approximate the conditional distribution can again be bound in terms of factorization characteristics (we discuss this in the Supplementary Materials).

\subsection{Statistical Complexity}

The number of samples needed to learn a given distribution $p$ on $\cX \times \cY$ (for the total variation distance $\tv$) without any structural assumption is proportional to $|\cX| \times |\cY|$. More precisely, one can show:
\begin{theorem}[see e.g. \cite{canonne2020short}] \label{thm:tv}
    There exists an estimator $\wh p$ and a constant $C > 0$ such that, for all  $\ve,\delta >0$,
$$
\sup_{p} p^{\otimes n}( \tv(\wh p,p) \leq \ve) \geq 1 - \delta
$$
whenever
\begin{equation}
n \geq C \frac{|\cX||\cY| + \log(1/\delta)}{\ve^2}. \label{eq:sc1}
\end{equation}
\end{theorem}
An estimator satisfying the above property is simply the histogram estimator.
Because the (hidden) factorization assumption transforms the learning task into $\ell$ independent (but unknown) learning tasks, one can intuit that the quantity $|\cX||\cY|$ in \eqref{eq:sc1} will be replaced by a much smaller, factored counterpart. To make this intuition more formal, we introduce the following complexity
\begin{equation}
    \label{eq:SC}
    \tag{SC}
    \chi(p) = \min_{\cF \in \cG(p)} \ell^2 \left\{ \log \ell  + \max_{j \in [\ell]} |\pa_j| \times  q_j \right\},
\end{equation}
where $\cG(p)$ is as defined at the start of the section. For $\omega \in (0,1)$, we let $\Sigma(\omega)$ be the set of all probability distributions $p$ on $\cX \times \cY$, such that, 
\begin{enumerate}
    \item[({\em i})] The marginal of $p$ on $\cX$ is the uniform distribution;
    \item[({\em ii})] For any factorization $\cF$, either $\cF$ belongs to $\cG(p)$, or $\tv(p_\cF,p) > \omega$. 
\end{enumerate}
Assumption ({\em i}) can be replaced by ``the marginal is lower-bounded by a positive constant'', though it makes the proof slightly more involved. Assumption ({\em ii}) is a usual assumption from the conditional independence testing literature, see e.g. \cite{canonne2018testing} or \cite{neykov2021minimax}.

\begin{theorem} \label{thm:sc} There exists a constant $C >0$ and an estimator $\wh p$ such that for all $\omega, \ve, \delta$ with 
$\ve < \omega/(\log_2(|\cX||\cY|)|\cX||\cY|)^{1/2}$, there holds, for any $p \in \Sigma(\omega)$
$$
p^{\otimes n}( \tv(\wh p,p) \leq \ve) \geq 1 - \delta
$$
as soon as
$$
n \geq C \frac{\chi(p) + \log(4/\delta)}{\ve^2}. \label{eq:sc2}
$$
\end{theorem}
The proof can be found in the Supplementary Materials. Thus, under the factorization hypothesis, the sample complexity does not scale with $|\cX| \times |\cY|$ as in \eqref{eq:sc1} but with $\chi(p)$ as introduced in \eqref{eq:SC}, 
which is always smaller than~\eqref{eq:sc1} (and often {\em exponentially smaller}). The strength of Theorem \ref{thm:sc} is that the estimator $\wh p$ does not depend on the knowledge of an optimal factorization (i.e. one minimizing \eqref{eq:SC}), but performs just as well as if it did. We show in the next section that NNs match this performance, at least in a controlled experimental setting. 


\section{Experiments} \label{sec:expe}

We explore the practical implications of our discrete factorized structure~\eqref{eq:prod} on the performance of MLPs. We link these performances to the parameters of the factorization and to the various hyperparameters of our models, highlighting that MLPs perform better under~\eqref{eq:prod} while requiring fewer parameters and at a lower computational cost.

\subsection{Experimental Design} \label{sec:expdesign}

We describe in this section the data generation process, namely our procedure for generating distributions $p(y,x)$ satisfying~\eqref{eq:prod}, and then describe our chosen MLP architecture.  The precise values of the parameters of our data generation process, of the model's hyperparameters and of their associated default values can be found in Table 1 in the Supplementary.

\paragraph{Data specification} 
Unless otherwise specified, we let the token spaces be $\cX = \cY \cong [4096]$; as $N$ and $M$ are equal to $2^{12}$, this choice allows for multiple possible factorizations.
Our data model depends on four parameters.

\begin{itemize}[leftmargin=.9cm]
  \item[(P1)] Input factors: $(p_i)_{i\in[k]} \in \bN^k$.
   \item[(P2)] Output factors: $(q_j)_{j\in[\ell]}\in \bN^\ell$
  \item[(P3)] Number of parents: $(|I_j|)_{j\in[\ell]} \in \N^\ell$.
\end{itemize}
Although this results in slightly less general factorizations (as all the $y$-factors have the same number of parents), making $|I_j|$ constant across $j\in[\ell]$ reduces randomness in the graph generation process. 
Notably, it turns $\chi$~\eqref{eq:SC} and $\bar\chi$~\eqref{eq:AC} into deterministic quantities.
Alternatively, we may want to consider:
\begin{itemize}[leftmargin=.9cm]
  \item[(P3')] Connectivity parameter: $\beta\in[0, 1]$.
\end{itemize}
For each tuple of integers $(i,j) \in [k] \times [\ell]$, we let $i$ belong to $I_j$, which is equivalent to the coordinates $x_i$ appearing in $\pa_j$, with probability  $\beta$. 
In other terms, $\beta$ controls the probability of activation of each edge in the bipartite graph linking the $x$-factors to the $y$-factors. 
Unless otherwise specified, we set ourselves in (P3) (rather than (P3')). 

\begin{itemize}[leftmargin=.9cm]
  \item[(P4)] Concentration parameter: $\alpha \in \R_+$.
\end{itemize}
We let the marginal distribution $p(x)$ be the uniform law on $\cX$ for all experiments.
We generate conditional distributions as follows: for each $j\in [k]$, and each point $\pa_j \in \prod_{i\in I_j} \cX_i$, we sample the vector $(p(y_j|\pa_j))_{y_j\in\cY_j}$ according to a Dirichlet distribution of parameter $(\alpha,\dots,\alpha) \in \bR^{q_j}$.
When $\alpha = 10^{-3}$, this leads to $(p(y_j|\pa_j))$ being very close in distribution to a Dirac, while $\alpha = 1$ leads to $(p(y_j|\pa_j))$ being more uniformly sampled across the simplex. 

\paragraph{Model and optimizer specification.}
To keep our experimental design simple, we use an MLP block and an optimizer that are fairly standard for LLMs. The exact specifications are detailed in the Supplementary Materials.
It leads to five more hyperparameters:
\begin{itemize}[leftmargin=.9cm]
  \item[(P5)] The embedding dimension: $d\in\N$;
  \item[(P6)] The hidden dimension: $h\in\N$;
  \item[(P7)] The number of layers: $L\in\N$.
  \item[(P8)] The initial learning rate: $\eta > 0$;
  \item[(P9)] The number of epochs: $T \in \bN$. 
\end{itemize}
We summarize in Table \ref{tab:1} of the Supplementary Materials our default hyperparameters.
By default, our figures are averaged over 100 independent runs.
In some settings, we wait until the training procedure converges to study the optimal point reached by the networks.
This requires a larger number of epochs, which we set to $T = 10^6$ (otherwise $T=10^3$).
For these experiments, we only average our figures over 10 runs.

\subsection{Single Pass Study}

\begin{figure}[ht]
    \centering
    \includegraphics{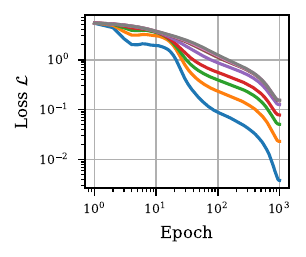}
    \raisebox{2.75em}{\includegraphics{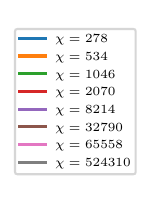}}
    \vspace{-1em}
    \caption{
        Value of the population loss ${\cal L}$~\eqref{eq:loss} in the single pass setting as a function of the number of epochs for various values of the statistical complexity parameter $\chi$~\eqref{eq:SC}.
        We obtain various values of $\chi$ by modifying the data model parameter (P1). 
    }
    \label{fig:iid}
\end{figure}

In our first experiment, we consider the typical setting of LLMs training: a single pass over the data.
In this setting, each update is done by sampling a batch of i.i.d. samples $(x^{(t)}, y^{(t)}) \sim p$ drawn according to a data distribution generated as described in Subsection~\ref{sec:expdesign}.
We use batches of size $n = 8096$, we set all parameters and hyperparameters to default values except for the input factorization $(p_i)_{i \in [\ell]}$, and we study how the test loss evolves with respect to the number of epochs.\footnote{More specifically, we let $(p_i)$ be $(2)_{i\in[12]}$, $(4)_{i\in[6]}$, $(8)_{i\in[4]}$, $(16)_{i\in[3]}$, $(64)_{i\in[2]}$, or $(4096)_{i\in[1]}$.}
We do not vary output factorization $(q_j)$ to avoid confounding factors caused by changes in the distribution  of the entropy of $p(y|x)$.\footnote{
    Indeed, we observe that neural networks performance is impacted by the entropy of the target conditional distributions.
}
In this setting, the embedding $e:\X\to\R^d$ is learned.

We find that the learning speed  is correlated with the statistical complexity $\chi$~\eqref{eq:SC}, as highlighted in Figure~\ref{fig:iid}. 
This suggests that the MLP is able to leverage the implicit product form of the target distribution to improve its learning.

\subsection{Compression Study}\label{sec:comp}

\begin{figure}[ht]
    \centering
    \includegraphics{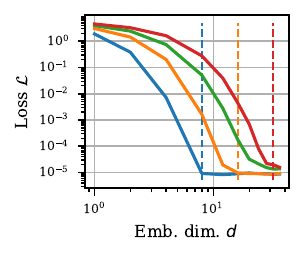}
    \raisebox{3em}{\includegraphics[width=8em]{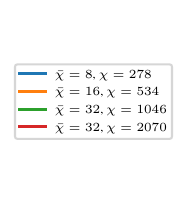}}
    \vspace{-1em}
    \caption{
        Value of the population loss ${\cal L}$~\eqref{eq:loss} in the compression setting after $T=10^6$ epochs of training as a function of the embedding dimension $d$ for $|I_j| \in \{1, 2, 3, 4\}$ (P3).
        The dashed lines indicate $d = \bar\chi$.
    }
    \label{fig:compression}
\end{figure}

\begin{figure}[ht]
    \centering
    \includegraphics{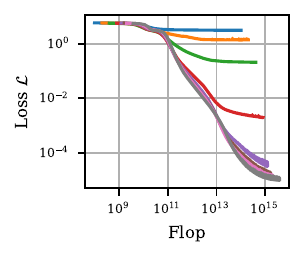}
    \raisebox{2.5em}{\includegraphics{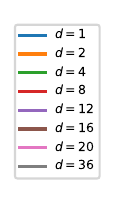}}
    \vspace{-1em}
    \caption{
        Value of the population loss ${\cal L}$~\eqref{eq:loss} in the compression setting as a function of the compute for various values of the embedding dimension $d$ (P5).
        Note that $\bar \chi = 16$~\eqref{eq:AC} with our default experimental parameters.
    }
    \label{fig:compression-flop}
\end{figure}

In this setting, we assume that the batch size $n$ approaches to infinity, so that we directly observe the conditional probabilities $p(y|x)$ for all values of $x$, and can optimize directly on the population loss ${\cal L}$. 
Once again, the embedding $e:\X\to\R^d$ is learned.
We know from the theoretical analysis in Subsection~\ref{sec:approx} that choosing $d \geq \bar\chi$ enables $\widehat p$ to match the ground-truth distribution $p$ for a well-chosen set of weights defining $u$, $e$ and $F$.
However, it is well-known that the class of functions that a neural network can represent is much bigger than the class of functions that are actually reachable after a reasonable number of gradient updates \citep{zhang2017understanding}, hence the need for experimental validation of the impact of both the complexity $\bar\chi$ associated with a probability distribution satisfying~\eqref{eq:prod} and the model's embedding dimension $d$.

Our findings are presented in Figure~\ref{fig:compression} and~\ref{fig:compression-flop}.
In Figure~\ref{fig:compression}, we observe that letting $d \geq \bar\chi$~\eqref{eq:AC} leads to the network learning the problem well, with the loss saturating close to machine precision,  which supports our claims from  Section~\ref{sec:approx}.
When $d \leq \bar\chi$, the loss also depends on the statistical complexity $\chi$~\eqref{eq:SC}. 
This is not surprising, as it is also true of the global minimum over all functions taking the form of~\eqref{eq:functionalform}. 
In Figure~\ref{fig:compression-flop}, we see that the loss initially follows a power law as a function of the compute, before eventually saturating due to limitations imposed by the embedding dimension $d$.
 
\subsection{Generalization Study} \label{sec:genstu}

\begin{figure}[ht]
    \centering
    \includegraphics{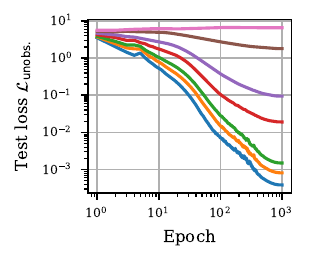}
    \raisebox{2.5em}{\includegraphics{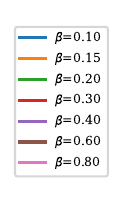}}
    \vspace{-1em}
    \caption{
        Value of the loss ${\cal L}_{\textrm{unobs.}}$ on unobserved data in the generalization setting as a function of the number of epochs for various values of the connectivity parameter $\beta$ (P3').
    }
    \label{fig:filtration}
\end{figure}

\begin{figure}[ht]
    \centering
    \includegraphics{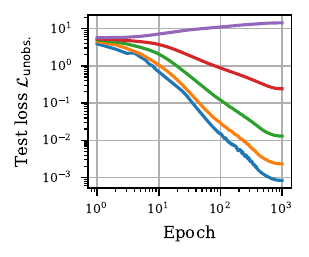}
    \raisebox{2.5em}{\includegraphics{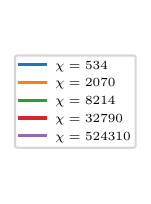}}
    \vspace{-1em}
    \caption{Value of the loss ${\cal L}_{\textrm{unobs.}}$ on unobserved data in the generalization setting as a function of the number of epochs for various values of the statistical parameter $\chi$~\eqref{eq:SC}.
    We obtain various values of $\chi$ by modifying (P1).
    } 
    \label{fig:generalization}
\end{figure}

In this setting, we study NNs' capacity for generalizing their learning of $p(y|x)$ to yet unseen inputs by leveraging the factorial nature of the task at hand.
More concretely, we assume that we only have access to a subset $\cX_{\textrm{obs}} \subset \cX$ of the data at training time.
This results in one additional experimental parameter:
\begin{itemize}[leftmargin=.9cm]
  \item[(P10)] The data split: $\gamma = \card{\cX_{\textrm{obs}}}/\card\cX$
\end{itemize}
By default we let $\gamma = 0.9 = 90\%$.
As in the compression setting, we set ourselves in the ideal scenario where we observe $p(y|x)$ for $x\in \cX_{\textrm{obs}}$, as if we were provided with very large batches.
We are interested in understanding how the network can generalize to unseen values of $x\in \cX\setminus \cX_{\textrm{obs}}$.
This setting falls into the scope of covariate shift \citep{sugiyama2007covariate} where the distribution of $x$ changes but the conditional distribution of $y|x$ remains the same, with the additional challenge here that the support of the shifted covariate is disjoint from the support of the observed one. 
To allow for generalization, the embedding is set to a fixed factorization-compatible embedding $\tilde e:\X\to\R^d$ (as defined in Subsection~\ref{subsec:embeddings}), with the embedding $e^i$ of each factor sampled as a centered isotropic Gaussian vector.

We observe in Figure~\ref{fig:filtration}  the effect of the connectivity of the factorization graph (see Figure~\ref{fig:dag}), driven by the connectivity parameter $\beta$ (P3'), on the  MLP's ability to generalize. 
The $y$-axis represents the loss ${\cal L}_{\textrm{unobs.}}$ on {\em unobserved data}.
As expected, when $\beta = 0$, there is no causal effect of $x$ on $y$, and, as $p(y|x)$ is thus constant across $\cX$, it is easy to generalize.
On the other end of the spectrum, when $\beta =1$, then $\pa_j = \cX$ for all $j$ and one cannot hope to generalize on unseen data.
We also observe that the generalization capability is linked to the statistical complexity (Figure~\ref{fig:generalization}).
This figure is obtained by varying the input factors as in the compression setting, with all the other values set to default. 

\begin{figure}[t]
    \centering
    \includegraphics{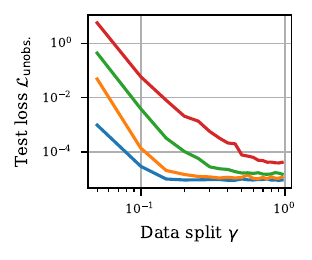}
    \raisebox{3em}{\includegraphics{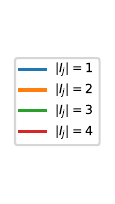}}
    \vspace{-1em}
    \caption{
        Value of the loss ${\cal L}_{\textrm{unobs.}}$ on unobserved data in the generalization setting after $T=10^6$ epochs of training as a function of the data split parameter $\gamma$ (P10) for various values of the number of parents $|I_j|$ (P3).
    }
    \label{fig:split}
\end{figure}

\begin{figure}[t!]
    \centering
    \includegraphics{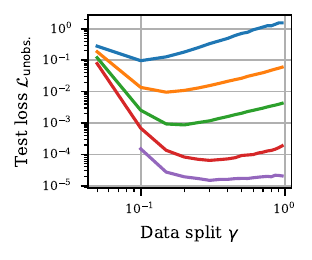}
    \raisebox{3em}{\includegraphics{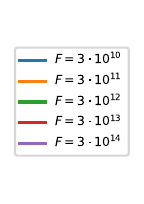}}
    \vspace{-1em}
    \caption{
        Value of the loss ${\cal L}_{\textrm{unobs.}}$ on unobserved data in the generalization setting as a function of the data split parameter $\gamma$ (P10) for various values of the number $F$ of flops.
    }
    \label{fig:split-flop}
    \includegraphics{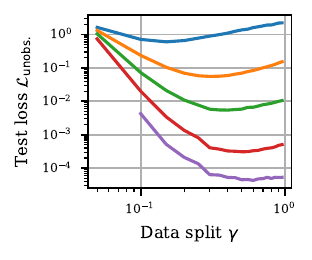}
    \raisebox{3em}{\includegraphics{figs/isoflop_leg.pdf}}
    \vspace{-1em}
    \caption{
        Same as Figure~\ref{fig:split-flop}, but with $|I_j|=3$  rather than  $|I_j|=2$ (P3).
    }
    \label{fig:split-flop-bis}
\end{figure}

Finally, we study the effect of the percentage of training data $\gamma = |\cX_{\mathrm{obs}}|/|\cX|$.
We observe that when dealing with highly sparse graphs, one does not need a lot of data to learn the underlying mechanisms that generated the data (Figure~\ref{fig:split}).
In Figures~\ref{fig:split-flop} and ~\ref{fig:split-flop-bis}, we also plot the ``isoflop'' graphs of the generalization loss ${\cal L}_{\textrm{unobs.}}$ as a function of the data split~$\gamma$ given different constraints on the number $F$ of floating-point operations (FLOP), which in turn constrain the number of training epochs $T$ since $F \propto \gamma T$ (as each epoch consists in a single pass over each training point).
Interestingly, we observe that for a given compute budget, it is better to train on the same data multiple times rather than to use new data.
This is particularly true when in the presence of highly structured data (compare Figure~\ref{fig:split-flop} with Figure~\ref{fig:split-flop-bis}).
This resonates with the recent findings of \citet{charton2024emergentpropertiesrepeatedexamples}.

\section{Discussion}

In this work, we introduced a new set of structural assumptions on discrete data. 
Through our theoretical analysis, we provided new insights into how these assumptions can simplify the learning process, and we showed through extensive experiments that NNs can exploit these hidden structures to enhance their performances. 

Our key takeaway is the presence of implicit relationships of the form 
\[
\mathcal{L} \propto \Lambda(\xi, G),
\]
where $\mathcal{L}$ is the loss, \(\xi\) is either the compute (e.g., Figure~\ref{fig:iid}), the model capacity (e.g., Figure~\ref{fig:compression}), or the number of data (e.g., Figure~\ref{fig:split}); and \(G\) represents the hidden ``graphical'' factorization underlying the data.
The function \(\Lambda\) is typically a decreasing function of \(\xi\) and an increasing function of complexity parameters of $G$, such as the number of factors (Figures~\ref{fig:iid},~\ref{fig:compression}, and~\ref{fig:generalization}), the connectivity of the graph (Figures~\ref{fig:filtration} and~\ref{fig:split}), or $\bar\chi$ and $\chi$ defined in \eqref{eq:AC} and \eqref{eq:SC}.



On a final note, and as hinted at in Subsection~\ref{subsec:embeddings}, we previse a possible strong link between our data model and {\em transfer learning}. Transfer learning focuses on transferring knowledge learned from one task to improve learning in another, often related but distinct task or domain \citep{pan2009survey}. Our framework could give a theoretical ground to explain the performance of transfer learning procedures on discrete data when compatible factorizations arise in different tasks, in the sense that a learned embedding adapted to a given task could become a factorization-compatible embedding for another subsequent task.

\bibliography{bibli}
\bibliographystyle{apalike}

\clearpage
\appendix
\thispagestyle{empty}
 \onecolumn

\section{Theory Supplement}

\begin{proof}[Proof of Theorem \ref{thm:ac}] In Section \ref{sec:theory}, we exhibited a solution to the matrix equation 
\begin{equation}\tag{\ref{eq:matrixEq}}
   (\log p(y|x))_{y,x} \in \R^{M\times N} = U^\top G + \text{column-wise constant}
\end{equation}
with $U\in \R^{Q\times M}$ and $G\in \R^{Q\times N}$ with as before $  Q = \sum_{j\in [\ell]} q_j$.
Another exact solution, which works as a dual of sorts to the first one, is given by
\[G' = \begin{bmatrix}
    (\ind{\pa_1})_{ x\in \cX} \\
    \vdots \\
    (\ind{\pa_\ell} )_{ x\in \cX}
    \end{bmatrix}   \in \R^{\bar P\times N}
\]
and
\[
U' = \begin{bmatrix}
    (\log p(y_1| z_1 ))_{ z_1\in \pa_1,y \in \cY} \\
    \vdots \\
    (\log p(y_\ell| z_\ell ))_{ z_\ell\in \pa_\ell,y \in \cY} 
    \end{bmatrix}   \in \R^{\bar P \times M},
\]
where $\bar P = \sum_{j\in[\ell]} \card{\pa_j}$, $(\ind{\pa_j})_{ x\in \cX} \in \R^{|\pa_j| \times N}$ is the matrix whose column indexed by $x\in \cX$ is the one-hot encoding of the combined factor $\pa_j$ of $x$, and the rows of the matrix $(\log p(y| z_j ))_{ z_j\in \pa_j,y \in \cY} \in \R^{|\pa_j| \times M}$ are indexed by the elements $z_j$ of $\pa_j$.
Combining those two solutions yields a third solution:
for each $j \in[\ell] $, let $G_j$ be the matrix $(\ind{\pa_j})_{ x\in \cX} \in  \R^{\card{\pa_j} \times N}$ if $\card{\pa_j} < q_j$, and let it be the matrix $(\log p(y_j|x ))_{y_j \in \cY_j, x\in \cX} \in \R^{q_j  \times N}$ otherwise.
Likewise, let $U_j$ be the matrix $(\log p(y_j| z_j ))_{ z_j\in \pa_j,y \in \cY} \in \R^{\card{\pa_j}  \times M} $ if  $\card{\pa_j} < q_j$, and let it be the matrix $ (\ind{y_j})_{ y\in \cY} \in \R^{q_j \times M} $ otherwise. 
Then the matrices
\[G'' = \begin{bmatrix}
    G_1 \\
    \vdots \\
    G_{\ell}
    \end{bmatrix}   \in \R^{ \sum_{j\in [\ell]}\min \{ |\pa_j|,  q_j \} \times N}
\]
and
\[
U'' = \begin{bmatrix}
    U_1 \\
    \vdots \\
    U_{\ell} 
    \end{bmatrix}   \in \R^{ \sum_{j\in [\ell]}\min \{ |\pa_j|,  q_j \} \times M}
\]
are solutions to \eqref{eq:matrixEq}, which shows that the embedding dimension $d$ can in fact be as small as  $\sum_{j\in [\ell]}\min \{ |\pa_j|,  q_j \}$ and still allow for a perfect representation of the conditional distribution.
\end{proof}

As alluded to in Section \ref{sec:theory}, the matrix $\Lambda = (\log p(y|x))_{y,x} $ having a small rank does not imply the existence of a non-trivial factorization.
Consider the following example: let $N,M \in \N$ and pick any $t_1,\ldots, t_N \in (0,1)$, as well as any $v\in (0,1)^M$ such that $\sum_{i=1}^M v_i =1$.
The matrix
\[
\Lambda := \begin{bmatrix}
   \log t_1& \ldots & \log t_N \\
    \log (1-t_1) + \log v_1& \ldots & \log (1-t_N) + \log v_1 \\
    \vdots & \ddots & \vdots \\
     \log (1-t_1) + \log v_M & \ldots & \log (1-t_N) + \log v_M
    \end{bmatrix}   \in \R^{ (M+1) \times N}
\]
defines a conditional (log-)distribution $\log p(y|x)$, and its rank is at most $3$. However, the distribution admits no non-trivial factorization for generic values of $t_1,\ldots, t_N$ and $v$.

Let us now consider the case of factorization-compatible embeddings, which was alluded to in the main text.
As before, $\Lambda$ must be expressed as a product $U^\top F(E)$ (up to some column-wise additive constant), but now  we have the constraint that the columns of $E$ are of the form
\begin{equation*}
    \tilde e_x = [E_1 \hdots E_k] 
    \begin{bmatrix}
    \ind{x_1} \\
    \vdots \\
    \ind{x_k}
    \end{bmatrix} = \sum_{i\in[k]} e^i(x_i),
\end{equation*}

To ease notations, let us set
\[
    P = \sum_{i\in[k]} p_i, \qquad \bar P = \sum_{j\in[\ell]} \card{\pa_j}.
\]
Assume that the matrix  $[E_1 \hdots E_k]$ is invertible, which is generically the case as soon as $d \geq P$ (as noted in Subsection~\ref{subsec:embeddings}), and let $T$ be its inverse, which maps $e_x$ to the product $(\ind{x_1}, \ldots,\ind{x_k} )\in \R^P$ of the one-hot encodings of the factors $x_i$ of $x$.
Let also 
$$\Xi: \R^P \to \R^{\bar P}$$ 
map $(\ind{x_1}, \ldots,\ind{x_\ell} )$ to the corresponding product $(\ind{\pa_1}, \ldots,\ind{\pa_\ell} ) $ of one-hot encoding of the parents.
The map $\Xi$ is non-linear, but it can be represented by a single feedforward layer (whose number of parameters is roughly $P\times \bar P$).
Then an exact solution can be expressed as
\begin{equation*}
 F: e\mapsto \Xi \cdot T \cdot e 
\end{equation*}
and
\[
    U = \begin{bmatrix}
    (\log p(y_1| z_1 ))_{ z_1\in \pa_1,y \in \cY} \\
    \vdots \\
    (\log p(y_\ell| z_\ell ))_{ z_\ell\in \pa_\ell,y \in \cY} 
    \end{bmatrix}   \in \R^{\bar P \times M},
\]
where the rows of the submatrix $(\log p(y_j| z_j ))_{ z_j\in \pa_j,y \in \cY} \in \R^{|\pa_j| \times M}$ are indexed by the elements $z_j$ of $\pa_j$.
As mentioned in the main text, this shows that we can again bound the number of parameters and the embedding dimension needed to approximate the conditional distribution in terms of factorization characteristics.

\begin{proof}[Proof of Theorem \ref{thm:sc}] In the proof, $C$ will denote a generic numeric constant, whose value might change from one line to another.
For a given factorization $\cF$, we denote by $q_{j}$ the marginal of $p$ over $\prod_{i \in I_j} \cX_i$, and $p_{j}$ the marginal of $p$ over $\prod_{i \in I_j} \cX_i \times \cY_j$.
We let $\wh p_{j}$ and $\wh q_{j}$ be their corresponding histogram estimators. Thanks to Theorem \ref{thm:tv}, it holds for all $j \in [\ell]$ and $\ve_j,\delta_j >0$,
$$
p^{\otimes n}(\tv(\wh q_{j}, q_{j}) \leq \ve_j) \geq 1-\delta_j,  
$$
whenever
$$
n \geq C \frac{|\pa_j| + \log(1/\delta_j)}{\ve_j^2}.
$$
In particular, under the previous condition and if $\ve_j < 1/|\cX| \leq 1/|\pa_j|$, it holds that $\wh q_j(\pa_j) > 0$ for all $\pa_j$. This allows us to define the estimator of the conditional probability $p(y_j|\pa_j)$ as
$$
\wh p(y_j|\pa_j) := \frac{\wh p_{j}(\pa_j,y_j)}{\wh q_{j}(\pa_j)},  
$$
and the subsequent estimator of $p_\cF$ as
$$
\wh p_\cF (x,y) := \frac{1}{|\cX|}  \prod_{j \in [\ell]} \wh p(y_j | \pa_j).
$$
Using again Theorem \ref{thm:tv}, one finds
$$
p^{\otimes n}(\tv(\wh p_j, p_j) \leq \ve_j) \geq 1-\delta_j,  
$$
whenever
$$
n \geq C \frac{|\pa_j| q_j + \log(1/\delta_j)}{\ve_j^2},
$$
Now notice that
\begin{align*}
\tv(\wh p_\cF,p_\cF) &\leq \frac{1}{|\cX|}\sum_{x \in \cX} \sum_{j\in[\ell]} \tv(p(\cdot|\pa_j), \wh p(\cdot|\pa_j)) \\
&= \frac{1}{|\cX|}\sum_{x \in \cX} \sum_{j\in[\ell]} \sum_{y_j \in \cY_j} |p(y_j|\pa_j)-\wh p(y_j|\pa_j)| \\
&= \frac{1}{|\cX|}\sum_{x \in \cX} \sum_{j\in[\ell]} \sum_{y_j \in \cY_j} \left|\frac{p_j(y_j,\pa_j)}{q_j(\pa_j)}-\frac{\wh p_j(y_j,\pa_j)}{\wh q_j(\pa_j)}\right| \\ 
&\leq \frac{1}{|\cX|}\sum_{j\in[\ell]} \sum_{x \in \cX} \sum_{y_j \in \cY_j} \left|\frac{p_j(y_j,\pa_j)}{q_j(\pa_j)}-\frac{\wh p_j(y_j,\pa_j)}{ q_j(\pa_j)}\right| + \left|\frac{1}{q_j(\pa_j)}-\frac{1}{\wh q_j(\pa_j)}\right|\wh p_j(y_j,\pa_j)
\end{align*}
Using that $q_j(\pa_j) = 1/|\pa_j|$ uniformly, one finds that
\begin{align*}
\tv(\wh p_\cF,p_\cF) 
&\leq \frac{1}{|\cX|}\sum_{j\in[\ell]} \sum_{x \in \cX} \sum_{y_j \in \cY_j} |\pa_j| \left|p_j(y_j,\pa_j)-\wh p_j(y_j,\pa_j)\right| + \left|\frac{1}{q_j(\pa_j)}-\frac{1}{\wh q_j(\pa_j)}\right| \wh p_j(y_j,\pa_j) \\ 
&= \frac{1}{|\cX|}\sum_{j\in[\ell]} |\pa_j| \times \frac{|\cX|}{|\pa_j|} \tv(\wh p_j, p_j) + |\pa_j| \times \frac{|\cX|}{|\pa_j|} \tv(\wh q_j,q_j) \\ 
&= \sum_{j \in [\ell]} \tv(\wh p_j, p_j)  + \tv(\wh q_j, q_j), 
\end{align*}
where we used that $\sum_{y_j \in \cY_j}\wh p_j(y_j,\pa_j) = \wh q_j(\pa_j)$ by definition. Letting $\ve_j = \ve / 2\ell $ and $\delta_j = \delta / 2\ell$, and using that $q_j \geq 1$ for all $j \in [\ell]$, one find that
$$
p^{\otimes n}(\tv(\wh p_\cF, p_\cF) \leq \ve) \geq 1-\delta
$$
as soon as
$$
n \geq  \frac{C \ell^2}{\ve^2}  \left( \max_{j \in [\ell]} |\pa_j| q_j + \log(2\ell/\delta) \right).
$$
 In particular, letting $\ell^* = \log_2 |\cY| \geq \ell$ and using the fact that  $|\pa_j| \leq |\cX|$, we have the coarser condition
$$
n \geq \frac{C (\ell^*) ^2}{\ve^2} \left( |\cX||\cY| + \log(2\ell^*/\delta)\right),
$$
that holds true uniformly for all factorizations $\cF$. Letting $\aleph$ be the cardinality of all possible factorizations, we find that
$$
p^{\otimes n}(\forall_{\cF}, \tv(\wh p_\cF, p_\cF) \leq \omega/4) \geq 1-\delta
$$
as soon as
$$
n \geq \frac{C  (\ell^*) ^2}{\omega^2} \left( |\cX||\cY| + \log(2\ell^* \aleph/\delta)\right).
$$
We then bound the cardinal $\aleph$. A factorization of $\cX$ is a bijection from $\cX$ to $\prod_{j \in [k]} [n_j]$ where $n_j$ are such that $n_1\times \dots\times n_k = |\cX|$. Because $\ell$ is smaller than $\log_2 |\cX|$, and because the $n_j$'s are smaller than $|\cX|$, there are less such factorizations than the number of functions from $\cX$ to $\cX^{N}$ where $N$ is the floor of $\log_2 |\cX|$, and this number is less than $|\cX|^{|\cX|\log_2 |\cX|}$. Likewise, the number of factorizations of $|\cY|$ is bounded by $|\cY|^{|\cY|\log_2 |\cY|}$. Finally, the number of DAG between two factorizations is less than $(2^{k})^\ell$ where $k$ is the number of factors of $\cX$ and $\ell$ the number of factors of $\cY$. Since $k \leq \log_2|\cX|$ and $\ell \leq \log_2 |\cY|$, one find that
$$
\aleph \leq |\cX|^{|\cX|\log_2 |\cX|} \times |\cY|^{|\cY|\log_2 |\cY|} \times |\cX|^{\log_2 |\cY|}. 
$$
The previous condition can thus be reframed, up to changing the value of $C$, as
$$
n \geq \frac{C}{\omega^2} \left( |\cX||\cY| \log_2 (|\cX||\cY| ) + \log(|\cX||\cY|/\delta)\right).
$$
We let $\cE$ be the event in which for all factorization $\cF$, $\tv(\wh p_\cF,p_\cF) \leq \omega/4$. In this event, one can tell whether a given $\cF$ is in $\cG(p)$ or not. Indeed, if $\cF$ is in $\cG(p)$, then, letting $\cO$ be the trivial factorisation, 
$$
\tv(\wh p_\cF,\wh p_\cO) \leq \tv(\wh p_\cF,p_\cF) + \tv(p,\wh p_\cO) \leq \omega/2,
$$
while if $\cF$ is not in $\cG(p)$,
$$
\tv(\wh p_\cF,\wh p_\cO) \geq \tv(p_\cF,p) - \tv(\wh p_\cF, p_\cF)  - \tv(\wh p_\cO, p) > \omega/2.  
$$ 
We can thus set
$$
\wh \cF \in \argmin \left\{ \ell^2 \max_{j \in [\ell]} |\pa_j| q_j + \ell^2 \log \ell~\middle|~  \tv(\wh p_\cF,\wh p)  \leq \omega/2\right\}.
$$
In case of ties, we consistently pick a value as long as it is possible, and we let $\cF^*$ be the value that $\wh \cF$ takes on the event $\cE$. It holds automatically that
$$
\cF^* \in \argmin \left\{ \ell^2 \max_{j \in [\ell]} |\pa_j| q_j + \ell^2 \log \ell~\middle|~  \cF \in \cG(p)\right\}.
$$
We finally let $\wh p := \wh p_{\wh \cF}$. We let $\cE'$ be the event on which $\tv(\wh p_{\cF^*},p) \leq \ve$. According to the above estimate, the event $\cE \cap \cE'$ has probability at least $1-2\delta$ as soon as
$$
n \geq  \frac{C}{\ve^2} \left( \chi(p) +  \log(2/\delta)\right) \bigvee \frac{C}{\omega^2} \left( |\cX||\cY| \log_2 (|\cX||\cY| ) + \log(|\cX||\cY|/\delta)\right).
$$
Using that $\ve^2 < \omega^2 \times ( |\cX||\cY| \log_2(|\cX||\cY|))^{-1}$, we find that the first term in the RHS is always the greatest, up to a numeric constant. Lastly, notice that on $\cE \cap \cE'$, $\tv(\wh p, p) \leq \ve$, ending the proof.
\end{proof}

\section{Experimental Supplement}

\paragraph{Model specification}
To keep our experimental design simple and in line with LLMs being one of our main sources of inspiration, we use the same MLP architecture as used for the feedforward layers of Mistral's open-source implementation of transformers at the time of writing. See the \texttt{one\_file\_ref} of \url{https://github.com/mistralai/mistral-inference} (commit 26a52a1).
It has three degrees of freedom:
\begin{itemize}[leftmargin=.9cm]
  \item[(P5)] The embedding dimension: $d\in\N$;
  \item[(P6)] The hidden dimension: $h\in\N$;
  \item[(P7)] The number of layers: $L\in\N$.
\end{itemize}
The full MLP corresponds to the composition of $L$ functions $F_i$ of the form:
\[
    F_i:z\in\R^d \mapsto z + W_{i,2}^\top\paren{\sigma\paren{W_{i,1}^{(i)}\frac{z}{\norm{z}}} \odot  W_{i,3} \frac{z}{\norm{z}}}
\]
where $W_{i,1}, W_{i,2}, W_{i,3} \in \R^{h\times d}$, $\sigma$ is the logistic function, and $\odot$ is the element-wise product.
This architecture was found to be more efficient in practice than purely vanilla MLP  \citep{he2015deep,ba2016layernormalization,shazeer2020gluvariantsimprovetransformer}.
Unless otherwise specified, we set $d = 32$, $h = 64$, and $L = 1$.
These choices are motivated by compute-optimal design experiments, to be found in the Supplementary Materials (Figures~\ref{fig:cal-dim} to~\ref{fig:cal-layer-bis}).

\paragraph{Optimizer specification}
To keep the optimization simple and mitigate the number of hyperparameters, we use the Adam optimizer as implemented in PyTorch and with default values for $\beta_1, \beta_2$.
Similarly, we initialize the NN's weights with PyTorch's default scheme. 
We distinguish between two settings regarding the number of epochs.

In the first setting, we focus on quantifying the speed of learning based on different underlying factorizations.
For this setting, we consider small numbers of epochs and use a cosine annealing learning rate schedule, defined as
\[
    \eta_t = \lambda_t\eta, \quad \text{where}\quad \lambda_t = \paren{\frac{\cos(\pi t / T) + 1}{2}} \in [0, 1],
\]
which leads to two additional hyperparameters:
\begin{itemize}[leftmargin=.9cm]
  \item[(P8)] The initial learning rate: $\eta > 0$;
  \item[(P9)] The number of epochs: $T \in \bN$. 
\end{itemize}
We set the initial learning rate to $\eta = 3\cdot 10^{-2}$ and the number of epochs $T$ to be $10^3$.
These choices are motivated by ablation studies to be found in the Supplementary Materials (Figure~\ref{fig:cal-lr}).
In this setting, our figures represent averages over 100 independent runs.

In our second setting, we wait until the training procedure converges to study the optimal point reached by the networks.
This requires a larger number of epochs, which we set to $T = 10^6$. 
For these experiments, we only average our figures over 10 runs.
Ablation studies in the Supplementary Materials (Figure~\ref{fig:cal-scheduler}) show that the cosine learning rate schedule is not optimal in this setting. 
On the one hand, choosing a large initial learning rate, such as $\eta = 3\cdot 10^{-2}$, allows for fast learning at the beginning but leads to instability towards the end of training.
On the other hand, choosing a small learning rate alleviates these instabilities but slows down the training at the beginning.
An analysis of the loss instabilities leads to the following ``custom'' scheduler:
\[
    \log (\eta_t) = \lambda_t \log(\eta) + (1-\lambda_t) \log(3\cdot 10^{-4}),
\]
This allows us to keep the same hyperparameters (P8) and (P9), which we set to $\eta = 3\cdot 10^{-2}$ and $T = 10^6$.

\paragraph{Computing archictures.}
All our experiments were run on a single V100 GPU.

\paragraph{Data calibration}
The data model, corresponding to (P1), (P2), (P3), (P3') and (P4) was chosen to allow for relatively large and sparse graphs (in the sense of Figure~\ref{fig:dag}), with $p(y|x)$ charging a small number of elements.
The other parameters were chosen based on compute optimal considerations, which are detailed below.

\renewcommand{\arraystretch}{1.7}
\begin{table}[h!]
    \centering
    \begin{tabular}{c|l|l|l|}
    \cline{2-4}
     & Name  & Notation & Default values  \\ \cline{2-4} \hhline{~|=|=|=|}
    (P1) & Input factors & $(p_i)_{i \in [k]}$ & $(2)_{i\in[12]}$ \\ \cline{2-4}  
    (P2)& Output factors  & $(q_j)_{i \in [\ell]}$ & $(8)_{j\in[4]}$\\ \cline{2-4} 
    (P3) & Number of parents  & $\card{I_j}$ & $2$  \\ \cline{2-4} 
    (P3') & Connectivity parameter & $\beta$ & --  \\ \cline{2-4} 
    (P4) & Concentration parameter & $\alpha$ & $10^{-1}$  \\ \cline{2-4} 
    (P5) & Embedding dimension  & $d$ & $64$ \\ \cline{2-4} 
    (P6) & Hidden dimension & $h$ & $128$ \\ \cline{2-4} 
    (P7) & Number of layers &$L$  & $1$ \\ \cline{2-4} 
    (P8) & Learning rate & $\eta$ & $3\cdot 10^{-2}$ \\ \cline{2-4} 
    (P9) & Number of epochs & $T$ & $\{10^3, 10^6\}$ \\ \cline{2-4} 
    (P10) & Pourcentage of seen data ({\em generalization setting}) & $\gamma$ & $0.9$ \\ \cline{2-4} 
    \end{tabular}
    \caption{The different parameters and hyperparameters from in Section \ref{sec:expe} and their default values used in the experiments.}
    \label{tab:1}
\end{table}

\paragraph{FLOPs}
In order to define the compute cost of our training, we use an idealized model for floating-point operations (FLOPs), consistent with common practice in similar studies \citep{scalinglawskaplan}. 

\begin{enumerate}
    \item Forward pass
    \begin{enumerate}
        \item The embedding layer is a look-up dictionary, resulting in 0 FLOPs.
        \item For each feedforward block, the total FLOPs $C$ across all layers is given by:
        \[
        C = 6 \times L \times d \times h.
        \]
        Indeed, for a single layer, we have three matrix multiplication of a vector of size $a$ by a matrix of size $a\times b$, where $\{a, b\} = \{d, h\}$, resulting in $2 \times a \times b$ FLOPs for each multiplication.
        \item The output layer, which is a linear layer, has a computational complexity of:
        \[
            C = 2 \times d \times M
        \]
    \end{enumerate}
\end{enumerate}

The backward pass is estimated to require approximately twice the compute of the forward pass \citep{scalinglawskaplan}.

\paragraph{Model calibration experiments}
Given the data model specified by (P1), (P2), (P3), (P4) from Table \ref{tab:1}, as well as the optimization model specified by $T = 10^3$ and $\eta \in \{10^{-1}, 10^{-2}, 10^{-3}\}$, we optimize the model parameters, $d$, $h$ and $L$, greedily one after the other.
We start with $h = 4d$, and $L=1$, leading to Figure \ref{fig:cal-dim}. 
This figure suggests setting $d \in \{64, 128\}$. 
Between these two options, we choose $d=64$ as it reduces the compute cost of our experiments (P5).
From there, Figure \ref{fig:cal-ffn} suggests choosing $h \in \{2d, 3d, \ldots\}$.
Between these options, we choose $h = 2d$ to save compute (P6).
Finally, Figure \ref{fig:cal-layer} explains our choice of $L = 1$ (P7).
Each figure was computed as an average over 100 independent runs.

It should be noted that in our calibration experiments, we choose $N=M=4096$, which is much bigger than $d=64$.
As a consequence, $L$ and $h$ do have a significant impact on the total number of FLOPs.
In order to show a bigger difference in the choice of these parameters, Figures \ref{fig:cal-ffn-bis} and \ref{fig:cal-layer-bis} consider the case where $N = 360$ and $M = 36$.

\begin{figure}[ht]
    \centering
    \includegraphics{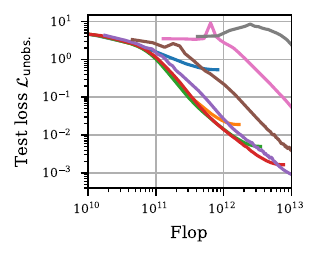}
    \includegraphics{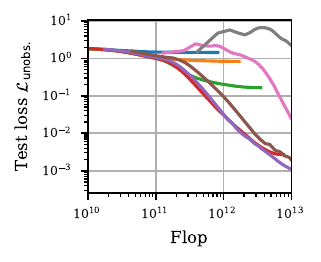}
    \raisebox{2.5em}{\includegraphics{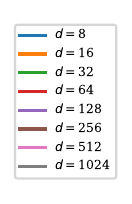}}
    \vspace{-1em}
    \caption{
        (Left) Effect of the embedding dimension $d$ on the test loss with respect to the number of FLOPs in the generalization setting with (P1), (P2), (P3), (P4) from Table \ref{tab:1}, $T=10^3$, $\eta = 10^{-2}$, which is best among the set $\{10^{-1}, 10^{-2}, 10^{-3}\}$, and $h=4d$, $L = 1$. 
        (Right) Same figure yet with $(p_i) = (8)_{i\in[4]}$ and $(q_j) = (4096)_{j\in[1]}$. 
        We observe that choosing $d=64$ yields good results in both experiments.
        This explains our choice of (P5).
    }
    \label{fig:cal-dim}
\end{figure}

\begin{figure}[ht]
    \centering
    \includegraphics{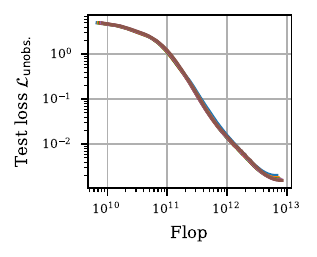}
    \includegraphics{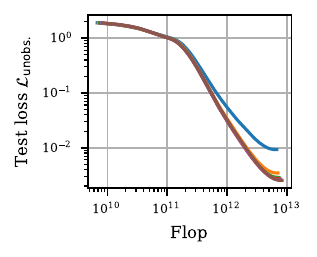}
    \raisebox{2.5em}{\includegraphics{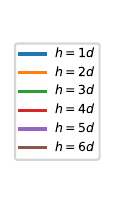}}
    \vspace{-1em}
    \caption{
        (Left) Effect of the hidden dimension $h$ on the test loss with respect to the number of FLOPs in the generalization setting with (P1), (P2), (P3), (P4), (P5) from Table \ref{tab:1}, $T=10^3$, $\eta = 10^{-2}$, which is best among the set $\{10^{-1}, 10^{-2}, 10^{-3}\}$, and $L = 1$. 
        (Right) Same figure yet with $(p_i) = (8)_{i\in[4]}$ and $(q_j) = (4096)_{j\in[1]}$. 
        The choice of $h = 2d$ (P6) yields good results.
    }
    \label{fig:cal-ffn}
\end{figure}

\begin{figure}[ht]
    \centering
    \includegraphics{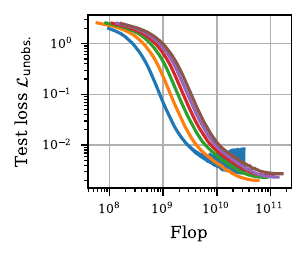}
    \includegraphics{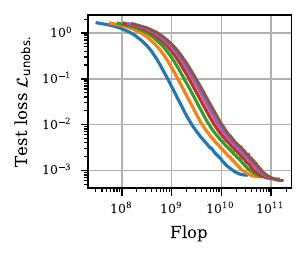}
    \raisebox{2.5em}{\includegraphics{figs/cal/ffn_leg.pdf}}
    \vspace{-1em}
    \caption{
        (Right) Same figure as Figure \ref{fig:cal-ffn}, yet with $(p_i) = (2, 2, 2, 3, 3, 5)$, and $(q_j) = (2, 2, 3, 3)$.
        (Left) Same figure yet with $(p_i) = (3, 4, 5, 6)$, and $(q_j) = (36)$
        This shows that the choice of $h = 2d$ is a decent one in general.
        Although for simple problems where $\bar\chi$ \eqref{eq:AC} is much smaller than $d = 64$, we observe that $h = d$ is more compute-efficient.
    }
    \label{fig:cal-ffn-bis}
\end{figure}

\begin{figure}[ht]
    \centering
    \includegraphics{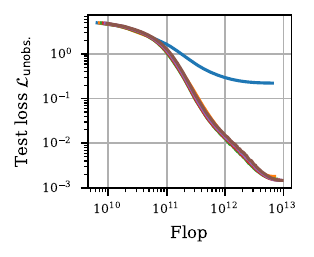}
    \includegraphics{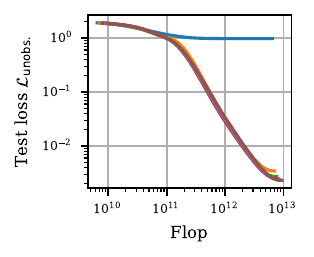}
    \raisebox{2.5em}{\includegraphics{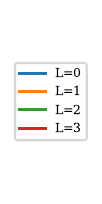}}
    \vspace{-1em}
    \caption{
        (Left) Effect of the number of layers $L$ on the test loss with respect to the number of FLOPs in the generalization setting with (P1), (P2), (P3), (P4), (P5), (P6) from Table \ref{tab:1}, $T=10^3$ and $\eta = 10^{-2}$ which is best among the set $\{10^{-1}, 10^{-2}, 10^{-3}\}$. 
        (Right) Same figure yet with $(p_i) = (8)_{i\in[4]}$ and $(q_j) = (4096)_{j\in[1]}$. 
        The choice of $L = 1$ (P7) yields good results.
    }
    \label{fig:cal-layer}
\end{figure}

\clearpage

\begin{figure}[ht]
    \centering
    \includegraphics{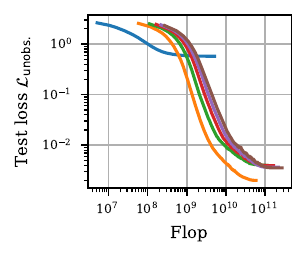}
    \includegraphics{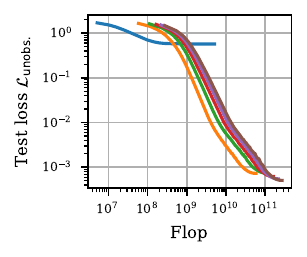}
    \raisebox{2.5em}{\includegraphics{figs/cal/layer_leg.pdf}}
    \vspace{-1em}
    \caption{
        (Right) Same figure as Figure \ref{fig:cal-layer}, yet with $(p_i) = (2, 2, 2, 3, 3, 5)$, and $(q_j) = (2, 2, 3, 3)$.
        (Left) Same figure yet with $(p_i) = (3, 4, 5, 6)$, and $(q_j) = (36)$
        This supports the choice of $L = 1$.
    }
    \label{fig:cal-layer-bis}
\end{figure}

\paragraph{Optimization calibration experiments}
In experiments to quantify the speed of learning, we set $T = 10^3$.
Figure~\ref{fig:cal-lr} shows that choosing $\eta = 3\cdot 10^{-2}$ yields good results in this setting.
In experiments to inspect NNs at convergence, we set $T=10^6$.
Figure \ref{fig:cal-scheduler} shows that the cosine learning-rate schedule is not optimal in our setting, as well as the performance of our custom scheduler:
\begin{equation}
    \label{eq:scheduler}
    \log (\eta_t) = \lambda_t \log(3\cdot 10^{-2}) + (1-\lambda_t) \log(3\cdot 10^{-4}), \qquad \text{where} \qquad \lambda_t = \frac{\cos(\pi t / 10^6) + 1}{2}.
\end{equation}
This is the learning rate scheduler we use for these long runs.

\begin{figure}[ht]
    \centering
    \includegraphics{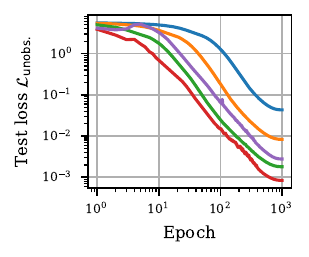}
    \includegraphics{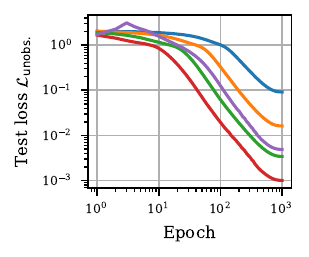}
    \raisebox{2.5em}{\includegraphics{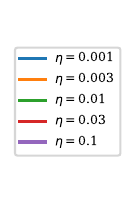}}
    \vspace{-1em}
    \caption{
        (Left) Effect of initial learning rate of the cosine schedule on the test loss with respect to the number of epochs in the generalization setting with (P1), (P2), (P3), (P4), (P5), (P6), (P7) from Table \ref{tab:1}, $T=10^3$. 
        (Right) Same figure yet with $(p_i) = (8)_{i\in[4]}$ and $(q_j) = (4096)_{j\in[1]}$. 
        The choice of $\eta = 3\cdot 10^{-2}$ yields good results in both experiments explaining our choice of (P9) given (P8).
    }
    \label{fig:cal-lr}
\end{figure}

\begin{figure}[ht]
    \centering
    \includegraphics{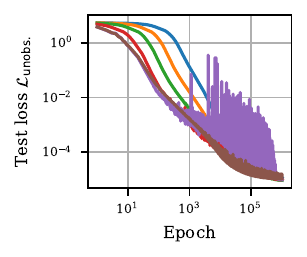}
    \includegraphics{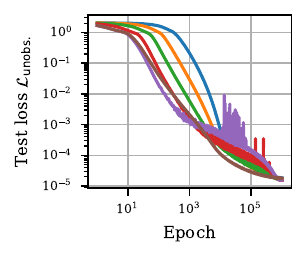}
    \raisebox{2.5em}{\includegraphics{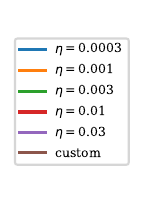}}
    \vspace{-1em}
    \caption{
        Same figure as Figure \ref{fig:cal-lr} yet with $T = 10^6$.
        Instabilities appear after around $10^3$ epochs.
        This suggests that we should modify the scheduler to decrease the learning rate earlier, for example with the ``custom'' scheduling of Eq. \eqref{eq:scheduler} (brown plot).
    }
    \label{fig:cal-scheduler}
\end{figure}

\clearpage 
\paragraph{Additional figures}
In addition to the calibration experiments in this appendix and  the figures in the main text, we report some additional experiments here.
The following figures are the results of averaging over 10 runs, rather than 100 runs.
Figure \ref{fig:compression-bis} repeats the compression study for another choice of experimental parameters.
Figure \ref{fig:emb-dim} investigates the interplay between the approximation complexity $\bar\chi$ and the embedding dimension $d$ in the generalization setting, rather than in the compression setting.
Figure \ref{fig:split-bis} provides an alternative visualization of the results presented in  Figure \ref{fig:split-flop} and illustrates the effect of the data split $\gamma$ and the number of flops on the generalization loss ${\cal L}_{\textrm{unobs.}}$ in the generalization setting for various factorizations.
Note that in order to save compute, Figures \ref{fig:split}, \ref{fig:split-flop} and \ref{fig:split-bis} were made with $d = 32$ (P5) instead of $d=64$.

\begin{figure}[ht]
    \centering
    \includegraphics{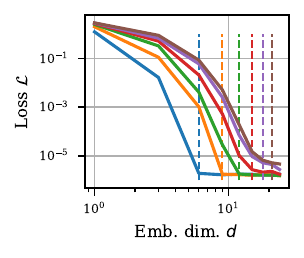}
    \raisebox{2.5em}{\includegraphics{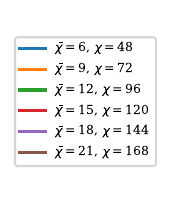}}
    \includegraphics{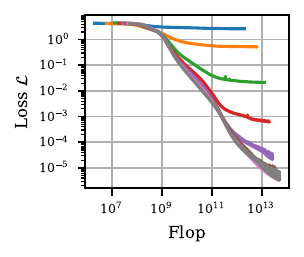}
    \raisebox{2.5em}{\includegraphics{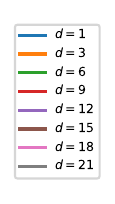}}
    \vspace{-1em}
    \caption{
        Same graphs as Figure \ref{fig:compression} (left) and Figure \ref{fig:compression-flop} (right), except that $(q_j) = (8)_{j\in[3]}$, $(p_i) = (x)_{i\in[4]}$, and $|I_j|=1$, with $x \in [2,7]$ for the plot on the left and $x=6$ for the plot on the right (which results in $\bar\chi = 18$).
    }
    \label{fig:emb-dim}
\end{figure}

\begin{figure}[ht]
    \centering
    \includegraphics{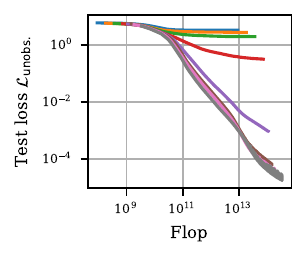}
    \includegraphics{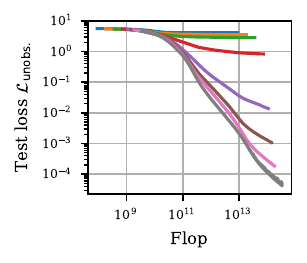}
    \raisebox{2.5em}{\includegraphics{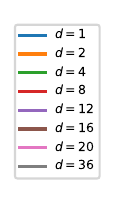}}
    \vspace{-1em}
    \caption{
        (Left) Same figure as Figure \ref{fig:compression-flop} yet in the generalization setting and with $T = 10^5$. (Right) Same figure yet with $|I_j| = 3$ (P3), which sets the compression complexity to $\bar\chi = 32$, instead of $\bar\chi=16$ \eqref{eq:AC}
    }
    \label{fig:compression-bis}
\end{figure}

\begin{figure}[ht]
    \centering
    \includegraphics[width=.28\linewidth]{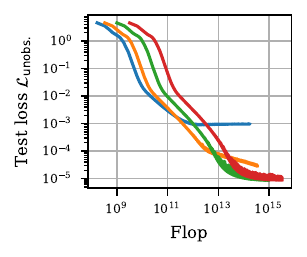}
    \includegraphics[width=.28\linewidth]{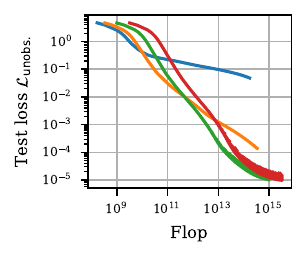}
    \includegraphics[width=.28\linewidth]{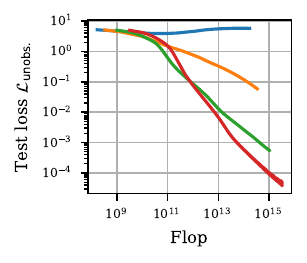}
    \raisebox{2em}{\includegraphics{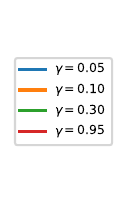}}
    \vspace{-1em}
    \caption{
           (Middle) Value of the loss ${\cal L}_{\textrm{unobs.}}$ on unobserved data in the generalization setting as a function of the number of FLOPs for various data split parameter $\gamma$ (P10). Same graph with $|I_j| = 1$ (P3) on the left and $|I_j| = 4$ (P3) on the right.
    }
    \label{fig:split-bis}
\end{figure}


\end{document}